\documentclass[10pt,twocolumn,letterpaper]{article}

\usepackage[pagenumbers]{cvpr} 

\usepackage{graphicx}
\usepackage{amsmath}
\usepackage{amssymb}
\usepackage{booktabs}
\usepackage{epsfig}
\usepackage{graphicx}
\usepackage{amsmath}
\usepackage{amssymb}
\usepackage{amsfonts}
\usepackage{adjustbox}
\usepackage{multicol}
\usepackage{multirow}
\usepackage{booktabs}
\usepackage{rotating}
\usepackage[usenames,dvipsnames]{color}
\usepackage{amsmath}
\usepackage{bbold}
\usepackage{bm}
\newtheorem{theorem}{Theorem}

\newtheorem{proof}{Proof}

\usepackage[pagebackref,breaklinks,colorlinks]{hyperref}

\usepackage[capitalize]{cleveref}
\crefname{section}{Sec.}{Secs.}
\Crefname{section}{Section}{Sections}
\Crefname{table}{Table}{Tables}
\crefname{table}{Tab.}{Tabs.}

\def\cvprPaperID{1673} 
\def\confName{CVPR}
\def\confYear{2022}

\begin{document}

\title{Clue Me In: Semi-Supervised FGVC with Out-of-Distribution Data}

\author{Ruoyi Du, Dongliang Chang, Zhanyu Ma$^*$, Yi-Zhe Song, and Jun Guo
}
\maketitle

\let\thefootnote\relax\footnotetext{R. Du, D. Chang, Z. Ma, Y.-Z. Song, and J. Guo are with the Pattern Recognition and Intelligent System Laboratory, School of Artiﬁcial Intelligence, Beijing University of Posts and Telecommunications, Beijing 100876, China. E-mail:\{beiyoudry, changdongliang, mazhanyu, guojun\}@bupt.edu.cn}
\let\thefootnote\relax\footnotetext{Y.-Z. Song is with SketchX, CVSSP, University of Surrey, London, United Kingdom. E-mail: y.song@surrey.ac.uk\\
\\
(Corresponding author: Zhanyu Ma)}

\begin{abstract}
   Despite great strides made on fine-grained visual classification (FGVC), current methods are still heavily reliant on fully-supervised paradigms where ample expert labels are called for. Semi-supervised learning (SSL) techniques, acquiring knowledge from unlabeled data, provide a considerable means forward and have shown great promise for coarse-grained problems. However, exiting SSL paradigms mostly assume in-distribution (\emph{i.e.}, category-aligned) unlabeled data, which hinders their effectiveness when re-proposed on FGVC. In this paper, we put forward a novel design specifically aimed at making out-of-distribution data work for semi-supervised FGVC, \emph{i.e.}, to ``clue them in''. We work off an important assumption that all fine-grained categories naturally follow a hierarchical structure (\emph{e.g.}, the phylogenetic tree of ``Aves'' that covers all bird species). It follows that, instead of operating on individual samples, we can instead predict sample \textit{relations} within this tree structure as the optimization goal of SSL. Beyond this, we further introduced two strategies uniquely brought by these tree structures to achieve inter-sample consistency regularization and reliable pseudo-relation. Our experimental results reveal that (i) the proposed method yields good robustness against out-of-distribution data, and (ii) it can be equipped with prior arts, boosting their performance thus yielding state-of-the-art results. Code is available at \url{https://github.com/PRIS-CV/RelMatch}.
\end{abstract}
\vspace{-0.1cm}
\section{Introduction}

Progress on computer vision had been heavily reliant on having access to annotated data. This however represents a heavy constraint for the problem of fine-grained visual classification (FGVC), where labels can only come from experts, \emph{i.e.}, people who can tell the difference between a ``American Crow'' and ``Fish Crow''. This essentially means despite great strides made~\cite{xiao2015application,lin2015bilinear,fu2017look,wang2018learning,dubey2018pairwise,chen2019destruction,du2020fine,du2021progressive}, our understanding of FGVC largely remains limited to fully supervised paradigms on categories where labels had been expertly curated for, \emph{e.g.}, Bird~\cite{wah2011caltech}, Flower~\cite{Nilsback2008Automated}, Cars~\cite{krause20133d}.

Elsewhere for coarse-grain problems, significant progress has already been made to address this ``lack of label'' phenomenon, notably through the means of semi-supervised learning (SSL)~\cite{grandvalet2004semi}. It can however be argued that as important as SSL is, the option of falling back to fully supervised remains conceivable for coarse-grain problems, \emph{i.e.}, more human (non-expert) labels can always be sourced with effort (and is arguably common practice in industry). The same nonetheless does not hold for FGVC, \textit{regardless} of resource and effort -- ultimately, there simply may never be enough experts available to label. Indeed, having sufficient labelled training data for FGVC might be a false argument to start with, thus making SSL an even more significant endeavour for FGVC.

\begin{figure}[t]
\centering
\includegraphics[width=1\linewidth]{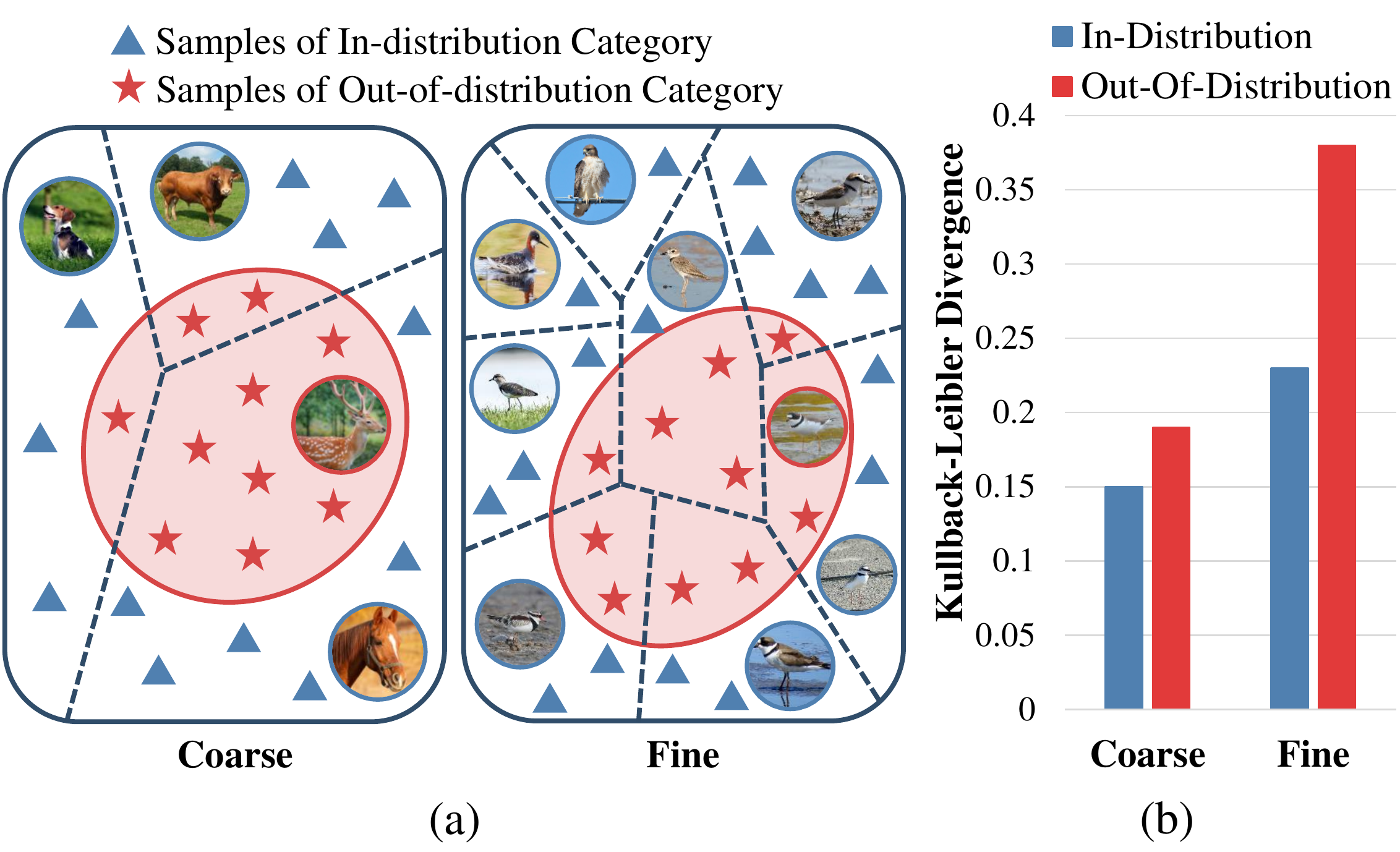}
\caption{(a) Illustration of feature spaces for coarse-grained classification and fine-grained classification. The \textcolor{Blue}{blue lines} indicate the decision boundaries and the \textcolor{Red}{red circle} represents the distribution of out-of-distribution category in the feature space. (b) The average Kullback-Leibler divergence of models' prediction probabilities to the probability distribution centers.}
\label{fig:figure1}
\vspace{-0.3cm}
\end{figure}

Semi-supervised learning in the context of FGVC is however non-trivial. The difficulty is largely owing to the overwhelming presence of out-of-distribution samples in unlabeled data. This is easily observable in Figure~\ref{fig:figure1}(a), where when compared with coarse-grain, the feature space of the fine-grained classifier is a lot more compact making new out-of-distribution unlabelled data (in red circle) much easier to be confused. To further verify our point, we simulate two situations where we make certain categories either in-distribution or out-of-distribution to the model and measure how dispersed the probability distributions of model predictions are via Kullback-Leibler Divergence. From Figure~\ref{fig:figure1}(b), we can conclude that (i) the model generally yields more inconsistent predictions for an out-of-distribution category, and (ii) this phenomenon is considerably more significant for the fine-grained model. This essentially renders most of the existing SSL work that rely on pseudo-labeling~\cite{lee2013pseudo,cascante2020curriculum,nassar2021all} or consistency regularization~\cite{bachman2014learning,laine2016temporal,sajjadi2016regularization,tarvainen2017mean} significantly less effective when re-purposed for FGVC. This is because that they mostly work with the assumptions that (i) acquiring in-distribution unlabeled data is relatively easy, and (ii) out-of-distribution is not as salient a problem given the coarse decision space.

In this paper, we put forward a novel design that specifically aims at making out-of-distribution data work for the problem of FGVC -- \emph{i.e.}, to ``clue them in''. To reason around the inherently tight decision space of FGVC is however not straightforward -- we need all the help we can get. For that, we take inspiration from~\cite{chang2021your} to utilize a label hierarchy of fine-grained categories (\emph{e.g.}, the phylogenetic tree of ``Aves''). As shown in Figure~\ref{fig:figure2}, the phylogenetic tree provides an underlying structure that all bird species naturally obey. For example, the model can not tell the unlabeled bird is a ``Collared Plover'', since it is totally out of the training label space. But the model can justifiably infer that it belongs to the same genus as ``Semipalmated Plover'' but not ``Semipalmated Plover'', \emph{i.e.}, owning the relation of ``Same Genus, Different Species'' with ``Semipalmated Plover'' as defined in Figure~\ref{fig:figure2}. It follows that our main innovation lies with predicting sample \textit{relations} within this hierarchy other than operating on individual samples.

More specifically, we utilize a simple multi-layer perceptron (MLP) that takes representations of sample pairs as input to work as a relation classifier. Predicting the relation of two birds can be regarded as a \emph{close-set} classification problem because the defined phylogenetic tree includes all expert bird species. Therefore, by replacing the instance-level prediction goal in previous pseudo-labeling techniques with this relation-based prediction, we construct a novel \textit{relation-based} pseudo-labeling strategy that importantly yields a common label space for both in-distribution and out-of-distribution data. 

We further propose two novel strategies taking full advantage of the tree-structured hierarchy to better align in-distribution and out-of-distribution samples. First, we re-purpose the concept of \emph{rooted triplet}~\cite{jansson2005rooted}, and introduce a \emph{triplet consistency} mechanism to enable inter-sample regularization of samples (see Figure~\ref{fig:figure2} for an example of rooted triplet). We conjecture that for any three unlabeled samples that form a \emph{rooted triplet}, exactly one of them will form consistent relations with the other two (see Figure~\ref{fig:figure3} (b) and Section~\ref{sub-sub-sec-rooted-triplet} for proof). Second, in the scenario where two of three leaf samples are labeled, we can further infer consistency between unlabeled relations and ground-truths relations (see Figure~\ref{fig:figure3} (d) and Section~\ref{sub-sub-sec-label-transfer}). Through this strategy we call label transfer, we significantly improve the quality of pseudo-relations learned which in turn helps with alignment.

We conduct experiments on semi-supervised FGVC benchmark datasets published by~\cite{su2021realistic} under both their conventional setting and the realistic setting. Compared with state-of-the-art semi-supervised and self-supervised methods, we achieve better or comparable performance. In addition, we show that (i) the proposed framework can be easily integrated with prior arts in a plug-and-play manner, further boosting their performance, and (ii) our relation-based approach exhibits good robustness against out-of-distribution data and is able to achieve performance gains when trained from scratch. Ablative study further shows our method can also work with categories that are completely out of our phylogenetic tree (\emph{e.g.,} the whole of ImageNet~\cite{russakovsky2015imagenet}).

\begin{figure*}[t]
\centering
\includegraphics[width=0.85\linewidth]{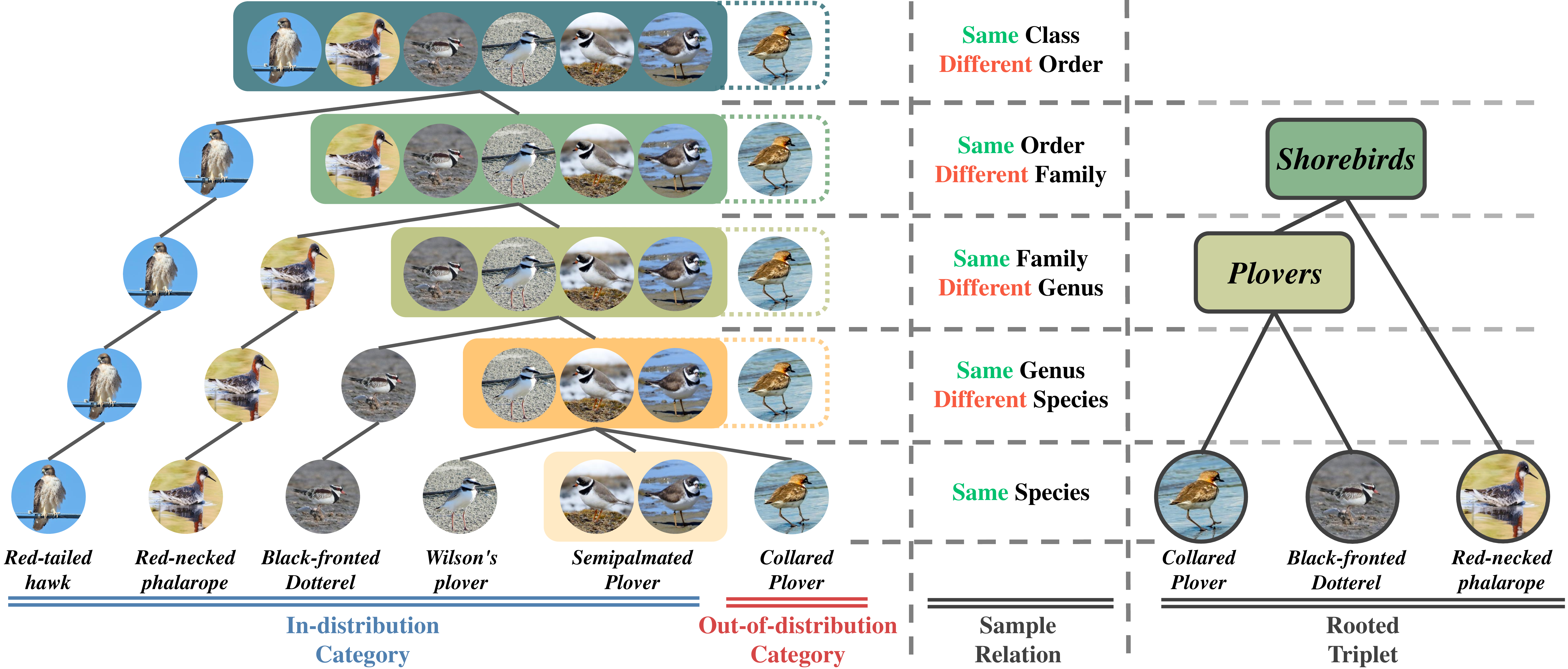}
\caption{Illustration of the underlying phylogenetic tree that both in-distribution samples and out-of-distribution samples should obey (shown on the left). Taking ``Aves'' classification as an example, according to the phylogenetic tree, the sample relations can be divided into $5$ types (shown in the middle). Three samples that form a rooted binary sub-tree called \emph{rooted triplet} (shown on the right).}
\label{fig:figure2}
\vspace{-0.3cm}
\end{figure*}

\section{Related Work}

\subsection{Fine-Grained Visual Classification}

Fine-grained visual classification (FGVC) tends to identify sub-categories that belong to a general class, \emph{e.g.}, distinguishing ``American Crow'' and ``Fish Crow''. With a history of almost two decades~\cite{biederman1999subordinate}, it has become a core problem of computer vision with many excellent works related (well summarized in a most recent survey~\cite{wei2021survey}).

Early works mostly heavily relied on dense bounding-box annotations~\cite{berg2013poof,chai2013symbiotic} to perform part detection, due to the relatively large intra-category variations. After that, some weakly-supervised methods emerged that only trained with category labels~\cite{zheng2017learning,fu2017look,wang2018learning}. However, due to the expertise needed for giving fine-grained labels, FGVC approaches are still limited to achieving further progress and also require higher labor costs in practice. Sparse data has become the bottleneck of FGVC in both academic research and industrial applications. Recently, researchers' interests have shifted to saving expert effort during training, \emph{e.g.}, visual recognition with small-sample~\cite{zhu2019image,li2020oslnet,li2020remarnet}, web-supervised learning~\cite{sun2019learning,zhang2020web}, and leveraging layer persons annotations~\cite{choudhury2021curious}. In this paper, we introduce a new lens to FGVC and propose a semi-supervised framework specifically aiming at FGVC tasks with out-of-distribution data.

\subsection{Semi-Supervised Learning}

Semi-supervised learning (SSL) is a popular topic for coarse-grained classification tasks and attract the attention of many researchers. Most related arts to date follows two tracks: (i) \textbf{Pseudo-labeling}~\cite{lee2013pseudo,cascante2020curriculum,nassar2021all} that utilizes model predictions to generate synthetic labels for unlabeled data as training targets, which is also regarded as implicit entropy minimization~\cite{grandvalet2004semi}, and (ii) \textbf{Consistency Regularization}~\cite{bachman2014learning,laine2016temporal,sajjadi2016regularization,tarvainen2017mean,miyato2019virtual} that applies stochastic disturbations on the input data or the model itself and encourages consistent model predictions can be given. Very recently, \cite{sohn2020fixmatch} and its variants~\cite{li2020comatch,hu2021simple} broke the boundary between two trends and proposed a simple yet effective new paradigm by combining pseudo-labeling and consistency regularization. 

Although prominent strides have been made, most of the SSL methods are evaluated and analyzed under an ideal setting that the categories of labeled data and the latent categories of unlabeled data are perfectly aligned. The work of~\cite{oliver2018realistic} first put forward this very question and pointed out that when it comes to realistic settings where out-of-distribution unlabeled data is present, performances of state-of-the-art techniques show drastic degradation with almost no exception. The very recent work of~\cite{su2021realistic} is most relevant to ours. Following~\cite{oliver2018realistic}, it evaluated state-of-the-art SSL methods in the FGVC task and obtained the same conclusion.

In this paper, as verified in Figure~\ref{fig:figure1} (b), we argue that the devil is in the prediction probability space -- it is hard for a model to make consistent predictions on samples with novel categories, which is obstructed to most existing SSL methods that operate on model predictions via whether pseudo-labeling, consistency regularization, or both. This is very intuitive since: (i) for pseudo-labeling methods, inconsistent predictions directly lead to meaningless artificial labels, and (ii) for consistency regularization methods, while model predictions sorely changing across samples, separately pursuing prediction consistency of individual sample also contribute less to discriminative knowledge mining.

\section{Methodology}

\subsection{Overview}

In this paper, focusing on tackling the semi-supervised fine-grained classification challenge with out-of-distribution data, we proposed RelMatch. It is composed by three main components: the feature exacter $\mathcal{F}(\cdot)$, the category predictor $\mathcal{P}_c(\cdot)$, and the relation predictor $\mathcal{P}_r(\cdot,\cdot)$. Then we can use $\mathcal{P}_c(\mathcal{F}(\cdot))$ for category classification of individual samples and $\mathcal{P}_r(\mathcal{P}_c(\mathcal{F}(\cdot)),\mathcal{P}_c(\mathcal{F}(\cdot)))$ for relation classification of sample pairs. Similar to existing methods, RelMatch simultaneously learns with labeled and unlabeled data. During the training phase, each batch both contains labeled samples $\{(x_i,y_i)\}_{i=1}^N \in X$ and unlabeled samples $\{u_j\}_{j=1}^{\mu N} \in U$, where $y_i$ are ground truth labels and $\mu$ is the ratio between the number of labeled samples and unlabeled samples. Generally, $\mu>1$ for better benefiting from the huge quantity of unlabeled samples. At each iteration, RelMatch optimizes with three losses: (i) a supervised category classification loss $\mathcal{L}_c$ to optimize $\mathcal{F}(\cdot)$ and $\mathcal{P}_c(\cdot)$, (ii) a supervised relation classification loss $\mathcal{L}_r$ to optimize $\mathcal{P}_r(\cdot,\cdot)$, and (iii) an unsupervised 
relation prediction loss $\mathcal{L}_{u}$ to optimize $\mathcal{F}(\cdot)$ and $\mathcal{P}_c(\cdot)$. Then the overall loss function is:
\begin{align}
    \mathcal{L}_{total} = \mathcal{L}_{c} + \mathcal{L}_{r} + \mathcal{L}_{u},
\end{align}
Specifically, $\mathcal{L}_c$ is a standard cross-entropy loss:
\begin{align}
    \mathcal{L}_{c} =\frac{-1}{N}\sum_{i=1}^{N} y_i \times \log(\mathcal{P}_c(\mathcal{F}(x_i))),
\end{align}
and the details of $\mathcal{L}_{r}$ and $\mathcal{L}_{u}$ are introduced in Section~\ref{sub-sec-pseudo-labeling} and Section~\ref{sub-sec-consistency-regularization} respectively.

\subsection{Relation-Based Pseudo-Labeling}\label{sub-sec-pseudo-labeling}

\subsubsection{Phylogenetic Tree} 

As Mentioned above, due to the relative compact feature space of the fine-grained classifier, out-of-distribution data, whose latent labels totally out of the training label space, hinder the effectiveness of previous SSL techniques. Instead of directly predicting the categories of individual samples, we need a flexible tool to ``clue'' all these data in. Following the label hierarchy in~\cite{chang2021your}, in this work, we leverage the phylogenetic tree\footnote{We adopt a specific kind of phylogenetic tree named ``cladogram'' whose branch lengths are uniform and do not represent time or relative amount of character change.} of living beings as a bridge between in-distribution and out-of-distribution data. It can be regarded as a rooted tree where categories at different classification levels (\emph{e.g.}, Class, Order, Family, Genus, and Species) are nodes with different depths. After that, as illustrated in Figure~\ref{fig:figure2}, relations of samples can be discretely defined within a \emph{close-set} (\emph{e.g.}, ``Same Family, Different Genus'', Same Genus, Different Species'', etc.). These relations can be expressed in a more formulaic manner. Borrowing the concept of lowest common ancestor (LCA) in graph theory, for a pair of samples $x_i$ and $x_j$ with category labels $y_i$ and $y_j$, their LCA can be formulated by $LCA(y_i,y_j)$. Then the pair relation can be represented by the depth of their LCA $Dep(LCA(y_i,y_j))$. In fact, the LCA depth of a pair of samples is proportional to their distance along with the route on the tree and can be regarded as a measurement of their similarity. In the subsequent parts of this article, we let $S(y_i,y_j)=Dep(LCA(y_i,y_j))$ for simplifying.

\subsubsection{Relation Prediction}

Just to reemphasize, the purpose of RelMatch is conducting individual category prediction, and all designs for relation prediction should serve this ultimate goal. Hence, we let the relation predictor $\mathcal{P}_r(\cdot,\cdot)$ take a pair of predictions from $\mathcal{P}_c(\cdot)$ as sample representation for input, which enables both the feature extractor $\mathcal{F}(\cdot)$ and the classifier $\mathcal{P}_c(\cdot)$ can be optimized by the relation-based supervision. Specifically, let $p=\mathcal{P}_c(\mathcal{F}(x))$ be the category prediction probability, where $p \in \mathbb{R}^C$ and $C$ is the number of categories. Then, with $p_i$ and $p_j$ as the input probability pair, the relation predictor first use vector outer product to model their element-wise correlation as $m_{i,j}=p_i^T \times p_j, m_{i,j} \in \mathbb{R}^{C \times C}$. After this, $m_{i,j}$ is multiplied by a learnable transfer matrix $w_t \in \mathbb{R}^{C \times C \times R}$ along with the softmax function to obtain the final relation prediction probabilities. The predicted results $r_{i,j}$ can be expressed as:
\begin{align}
     r_{i,j} & = \mathcal{P}_r(\mathcal{P}_c(\mathcal{F}(x_i)),\mathcal{P}_c(\mathcal{F}(x_j))) \nonumber \\
             & = Softmax(p_i^T \times p_j \times w_t),
\end{align}
where $r_{i,j} \in \mathbb{R}^R$ and $R$ is the number of discrete relations that equals to the depth of the phylogenetic tree. 

During the training phase, we simply combine the labeled samples $\{(x_i,y_i)\}_{i=1}^N \in X$ and its reverse version at the batch dimension to form a sample pair set $\{(x_i,y_i,x_{N-i},y_{N-i})\}_{i=1}^N \in B_x$. With $s_{i,j}=S(y_i,y_j)$ obtained by species labels as relation ground-truth labels, the supervised relation classification loss can be formulated as:
\vspace{-0.3cm}
\begin{equation}
    \mathcal{L}_{r} =\frac{-1}{N}\sum_{i=1}^{N} w_{s_{i,N-i}} \times s_{i,N-i} \times \log(r_{i,N-i}).
\end{equation}
Note that, when we randomly sample pairs from the phylogenetic tree, the occurrence probabilities of various relations are significantly different, which leads to a long-tailed distribution. To alleviate this problem, we re-weight the losses of pairs according to their ground-truth relations. And $\{w_i\}_{i=1}^R \in W$ is generated according to the statistical probability of each relation level's occurrence, where $\sum_{i=1}^R w_i = 1$.

\begin{figure*}[t]
\centering
\includegraphics[width=1\linewidth]{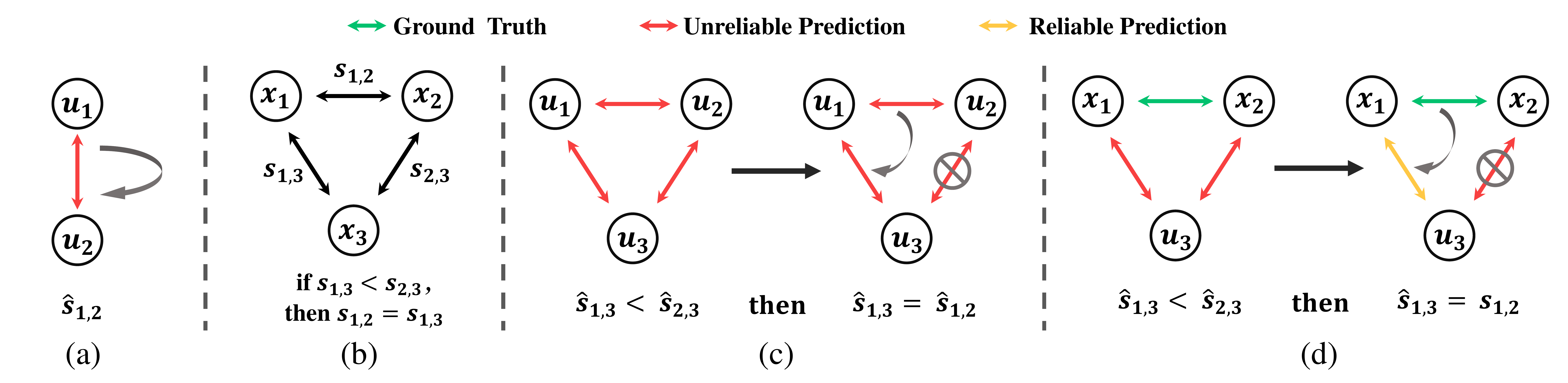}
\caption{(a) \textbf{Pseudo-labeling via relation prediction}: where the predicted relation $s_{1,2}$ is used as the pseudo-label directly. (b) Illustration of the \emph{triplet consistency} we define. (c) \textbf{Triplet consistency regularization}: where all three samples are unlabeled. (d) \textbf{Label Transfer}: where only one unlabeled sample participates, and the other two labeled data introduce the ground-truth relation.}
\label{fig:figure3}
\vspace{-0.3cm}
\end{figure*}
\vspace{-0.1cm}

\subsubsection{Naive Pseudo-Labeling via Relation Prediction}\label{relation_pl}

With predictions of pair relations, we can simply replace the instance-based prediction in previous SSL methods (\emph{e.g.}, pseudo-labeling~\cite{lee2013pseudo}) with the relation-based prediction (shown in Figure~\ref{fig:figure3}~(a)). It seems to be a straight yet intuitive solution that not only enables the model to better learn from out-of-distribution data but also keeps the merits of previous arts. For the relation-based pseudo-labeling, the optimization goal can be formalized as:
\begin{equation}~\label{equation3}
    \mathcal{L'}_{u}=\frac{-1}{\mu N}\sum_{j=1}^{\mu N}\sum_{k=1}^{\mu N}\tau_{j,k} \times w_{\hat{s}_{j,k}} \times \Hat{s}_{j,k} \times log(r_{j,k}),
\end{equation}
where $\hat{s}_{j,k}=argmax(r_{j,k})$ indicates the relation-level pseudo-labels, and $\tau_{j,k}=\hat{s}_{j,k}>t$ is used for sample selection with the confidence threshold $t$. 

For semi-supervised learning, it is all about learning the underlying structure of a large amount of unlabeled data~\cite{nair2019realmix}. So far, we only rely on the phylogenetic tree to build a common label space shared by in-distribution and out-of-distribution data. In the next part, we will take a further step and introduce how the tree structure of fine-grained categories can help us achieve better alignment.

\subsection{Triplet Consistency}\label{sub-sec-consistency-regularization}

\subsubsection{Consistency of Rooted Triplet}\label{sub-sub-sec-rooted-triplet}

A \emph{rooted triplet} is a distinctly leaf-labeled, binary, rooted, unordered tree with three leaves~\cite{jansson2005rooted} (illustrated in Figure~\ref{fig:figure2}). It has been well studied by algorithm researchers that, with a sufficient set of \emph{rooted triplets}, one can reconstruct the unique rooted tree containing all of them~\cite{aho1981inferring}. For an arbitrary leaf node triplet $\{A, B, C\}$ belonging to a rooted phylogenetic tree, if the lowest common ancestor (LCA) of $A$ and $B$ is a proper descendant of the LCA of $A$ and $C$, then the sub-tree they composed is a \emph{rooted triplet}. And there exists an obvious consistency that the LCA of $A$ and $C$ must also be the LCA of $B$ and $C$, which can be expressed as: 

\begin{theorem}\label{theorem1}
 For arbitrary leaf node triplet $\{A, B, C\}$ s.t. $Depth(LCA(A,B))>Depth(LCA(A,C))$, we have $LCA(A,C)=LCA(B,C)$.
\end{theorem}

\begin{proof}
Since $LCA(A,B)$ is an ancestor of $A$ and also a descendant of $LCA(A,C)$, we can draw the conclusion that $LCA(A,C)=LCA(LCA(A,B),C)$. Also, $LCA(A,B)$ is an ancestor of $B$, which leads to the equivalence that $LCA(LCA(A,B),C)=LCA(B,C)$. Thus, we can obtain Theorem~\ref{theorem1}.
\end{proof}

\subsubsection{Naive Consistency Regularization via Triplet Consistency}\label{triplet_cr}

With our notations, the \emph{triplet consistency} defined by Theorem~\ref{theorem1} can be formulated as $r_{j,l}=r_{k,l}$ for all $\hat{s}_{j,k}>\hat{s}_{j,l}$ (shown in Figure~\ref{fig:figure3}~(b)). Unlike previous consistency regularization methods that only rely on intra-sample consistency. The \emph{triplet consistency} enable interactions among samples, \emph{e.g.}, $r_{j,l}=r_{k,l}$ indicates the model predictions on $u_j, u_k$ should hold consistent relations with model prediction on $u_l$. In this way, we can also obtain a \emph{triplet consistency} regularization (shown in Figure~\ref{fig:figure3}~(c)) for unsupervised optimization as:

\begin{equation}
    \mathcal{L''}_{u}=\frac{1}{\mu N}\sum_{j=1}^{\mu N}\sum_{k=1}^{\mu N}\sum_{l=1}^{\mu N}\mathbb{1}(\hat{s}_{j,k}>\hat{s}_{j,l})\mathcal{H}(r_{j,l},\hat{s}_{k,l}),
\end{equation}

where $\hat{s}_{j,k}=argmax(r_{j,k})$, and $\mathcal{H}(\cdot)$ is an arbitrary function for distance measurement. And when $\mathcal{H}(\cdot)$ is in a re-weighted cross entropy form, \emph{i.e.}, the consistency regularization is achieved by giving pseudo-labels, it can be regarded as an improved version of Equation~\ref{equation3} by introducing interactions between samples. In the following discussion, we will use this form by default.

\subsubsection{Label Transfer via Triplet Consistency}\label{label_transfer}\label{sub-sub-sec-label-transfer}

For an unlabeled \emph{rooted triplet} $\{u_j,u_k,u_l\}$ with $\hat{s}_{j,k}>\hat{s}_{j,l}$, $\hat{s}_{k,l}$ is the pseudo-label for $r_{j,l}$ due to $r_{j,l}=r_{k,l}$. We notice that, with the \emph{triplet consistency}, quality of the pseudo-label is decided by qualities of both the relation prediction $\hat{s}_{k,l}$ and the relation comparison between $\hat{s}_{j,k}$ and $\hat{s}_{j,l}$, where still remains room for improvement. 

To approach the optimal solution, \emph{i.e.,} the proposed RelMatch, we re-sample the \emph{rooted triplet} with two labeled samples and one unlabeled sample as $\{x_j,x_k,u_l\}$. Then instead of unreliable prediction, the pseudo-label $\hat{s}_{k,l}$ is replaced by the ground-truth $s_{k,l}$ now. And once the relation comparison is correct, an exact pseudo-label can be given. In this way, the advantage brought by \emph{triplet consistency} is quite remarkable. It is all about building a bridge between labeled data and unlabeled data -- relation-based labels can be transferred from labeled pairs to unlabeled pairs with a simple relation comparison (shown in Figure~\ref{fig:figure3}~(d)). Ultimately, we reduce the difficulty of generating pseudo-labels from a multi-class classification problem to a binary classification problem. Our unsupervised optimization goal via label transfer is:
\begin{equation}
    \mathcal{L}_{u}=\frac{-1}{\mu N^3}\sum_{j=1}^{N}\sum_{k=1}^{N}\sum_{l=1}^{\mu N}\tau_{k,l} \times w_{s_{k,l}} \times s_{k,l} \times log(r_{j,l}),
\end{equation}
where $\tau_{k,l}=\mathbb{1}(E(s_{j,k})-E(s_{j,l})\geq1)$ serves as the condition function for sample selection, $E(s_{j,k})=\sum_{n=1}^R r_{j,k,n} \times n$ indicates the mathematical expectation of $s_{j,k}$, and $E(s_{j,l})$ is the same.

To summarize, RelMatch can be regarded as a combination of pseudo-labeling and consistency regularization with (i) relation-based pseudo-labels \emph{instead of} instance-based pseudo-labels, (ii) inter-sample consistency regularization \emph{instead of} intra-sample consistency regularization, and (iii) label transfer for pseudo-label generation \emph{instead of} directly predicting them.

\section{Experiments}

\subsection{Datasets}

We adopt two SSL-FGVC benchmark datasets Semi-Aves and Semi-Fungi released by~\cite{su2021realistic} for our comparison experiments and ablation studies. Each of them consists of three sub-sets: labeled samples $(x_i,y_i) \in X$, in-distribution unlabeled samples $u_j \in U_{in}$ which shares the same label space with $X$, and out-of-distribution unlabeled samples $u_k \in U_{out}$ which include categories that are novel but still belong to the phylogenetic tree. We experiment with $U_{in}+U_{out}$ as unlabeled data to simulate realistic applications where we cannot tell whether unlabeled data are in-distribution or not. Consistent with general SSL works, we also keep the setting that only considers $U_{in}$ as unlabeled data for unsupervised optimization. Statistics of two datasets are displayed in Table~\ref{dataset}.

\begin{table}[htbp]
\renewcommand{\arraystretch}{1.3}
\centering
\begin{adjustbox}{width=0.85\linewidth,center}
\begin{tabular}{c|c c c|c c c}
\toprule[1pt]
\multirow{2}{*}{Dataset} & \multicolumn{3}{c|}{Categories} &
\multicolumn{3}{c}{Images}
\cr\cline{2-7} & $X_{in}$ & $U_{in}$ & $U_{out}$ & $X_{in}$ & $U_{in}$ & $U_{out}$\\
\midrule[1pt]
Semi-Aves & $200$ & $200$ & $800$ & $6K$ & $27K$ & $122K$\\
Semi-Fungi & $200$ & $200$ & $1194$ & $4K$ & $13K$ & $65K$\\
\midrule[0.5pt]
\end{tabular}
\end{adjustbox}
\vspace{0.1cm}
\caption{Statistics of Semi-Aves and Semi-Fungi dataset.}
\label{dataset}
\end{table}
\vspace{-0.3cm}

\begin{table*}[t]
\begin{adjustbox}{width=1\linewidth,center}
\makeatletter\def\@captype{table}
\begin{minipage}[t]{0.5\linewidth}
\begin{adjustbox}{width=1\linewidth,center}
\begin{tabular}{c l|c|c|c}
\toprule[1pt]
\multirow{2}{*}{}&\multirow{2}{*}{\textbf{Method}} & \textbf{From Scratch} ($\%$) &
\textbf{From ImageNet} ($\%$) & \textbf{From iNat} ($\%$) 
\cr\cline{3-5} & & Top-1 / Top-5 & Top-1 / Top-5 & Top-1 / Top-5\\
\midrule[1pt]
&Baseline & $20.6\pm0.4$ / $41.7\pm0.7$ & $52.7\pm0.2$ / $78.1\pm0.1$ & $65.4\pm0.4$ / $86.6\pm0.2$\\
\midrule[0.5pt]
\midrule[0.5pt]
\multirow{6}{*}{\begin{rotate}{90}
$U_{in}$
\end{rotate}} 

&Pseudo-Label & $16.7\pm0.2$ / $36.5\pm0.8$ & $54.4\pm0.3$ / $78.8\pm0.3$ & $65.8\pm0.2$ / $86.5\pm0.2$\\
&CPL & $20.5\pm0.5$ / $41.7\pm0.5$ & $53.4\pm0.8$ / $78.3\pm0.5$ & $69.1\pm0.3$ / $87.8\pm0.1$\\
&FixMatch & $28.1\pm0.1$ / $51.8\pm0.6$ & \textcolor{Blue}{$\mathbf{57.4\pm0.8}$} / $78.5\pm0.5$ & \textcolor{Blue}{$\mathbf{70.2\pm0.6}$} / $87.0\pm0.1$\\
&Self-Training & $22.4\pm0.4$ / $44.1\pm0.1$ & $55.5\pm0.1$ / $79.8\pm0.1$ & $67.7\pm0.2$ / $87.5\pm0.2$\\

&MoCo & $28.2\pm0.3$ / $53.0\pm0.1$ & $52.7\pm0.1$ / $78.7\pm0.2$ & $68.6\pm0.1$ / $87.7\pm0.1$\\
&MoCo + Self-Training & $31.9\pm0.1$ / $56.8\pm0.1$ & $55.9\pm0.2$ / \textcolor{Blue}{$\mathbf{80.3\pm0.1}$} & $70.1\pm0.2$ / \textcolor{Blue}{$\mathbf{88.1\pm0.1}$}\\
\midrule[0.5pt]
&\textbf{RelMatch} & $34.5\pm0.2$ / $57.0\pm0.2$ & $54.4\pm0.2$ / $79.4\pm0.3$ & $69.5\pm0.2$ / $88.0\pm0.2$\\
&\textbf{RelMatch + MoCo} & \textcolor{red}{$\mathbf{37.9\pm0.1}$} / \textcolor{red}{$\mathbf{60.8\pm0.1}$} & $54.0\pm0.3$ / $79.9\pm0.2$ & $70.0\pm0.2$ / \textcolor{Blue}{$\mathbf{88.1\pm0.2}$}\\
&\textbf{RelMatch + FixMatch}& \textcolor{Blue}{$\mathbf{36.7\pm0.2}$} / \textcolor{Blue}{$\mathbf{60.4\pm0.6}$} & \textcolor{red}{$\mathbf{59.4\pm0.3}$} / \textcolor{red}{$\mathbf{81.1\pm0.5}$} & \textcolor{red}{$\mathbf{72.1\pm0.3}$} / \textcolor{red}{$\mathbf{88.8\pm0.4}$}\\
\midrule[0.5pt]
\midrule[0.5pt]
\multirow{9}{*}{\begin{rotate}{90}
$U_{in}+U_{out}$
\end{rotate}} 
&Pseudo-Label & $12.2\pm0.8$ / $31.9\pm1.6$ & $52.8\pm0.5$ / $77.8\pm0.1$ & $66.3\pm0.3$ / $86.4\pm0.2$\\
&CPL & $20.2\pm0.5$ / $41.0\pm0.9$ & $52.8\pm0.5$ / $77.8\pm0.1$ & $69.1\pm0.1$ / $87.6\pm0.1$\\
&FixMatch & $19.2\pm0.2$ / $42.6\pm0.6$ & $49.7\pm0.2$ / $72.8\pm0.5$ & $64.2\pm0.2$ / $84.5\pm0.1$\\
&Self-Training & $22.0\pm0.5$ / $43.3\pm0.2$ & \textcolor{Blue}{$\mathbf{55.5\pm0.3}$} / \textcolor{Blue}{$\mathbf{79.7\pm0.2}$} & $67.6\pm0.2$ / $87.6\pm0.1$\\
&MoCo & \textcolor{Blue}{$\mathbf{38.9\pm0.4}$} / $65.4\pm0.3$ & $51.5\pm0.4$ / $77.9\pm0.2$ & $67.6\pm0.1$ / $87.3\pm0.2$\\
&MoCo + Self-Training & \textcolor{red}{$\mathbf{41.2\pm0.2}$} / \textcolor{Blue}{$\mathbf{65.9\pm0.3}$} & $53.9\pm0.2$ / $79.4\pm0.3$ & $68.4\pm0.2$ / $87.6\pm0.2$\\
\midrule[0.5pt]
&\textbf{RelMatch} & $36.2\pm0.2$ / $57.3\pm0.3$ & \textcolor{red}{$\mathbf{55.6\pm0.4}$} / \textcolor{red}{$\mathbf{80.0\pm0.2}$} & \textcolor{Blue}{$\mathbf{69.4\pm0.1}$} / \textcolor{red}{$\mathbf{88.0\pm0.3}$}\\
&\textbf{RelMatch + MoCo} & \textcolor{red}{$\mathbf{41.2\pm0.2}$} / \textcolor{red}{$\mathbf{66.1\pm0.4}$} & $53.0\pm0.2$ / $79.2\pm0.2$ & \textcolor{red}{$\mathbf{69.6\pm0.3}$} / \textcolor{red}{$\mathbf{88.0\pm0.1}$}\\
&\textbf{RelMatch + FixMatch} & $32.3\pm0.3$ / $49.6\pm0.6$ & $55.1\pm0.3$ / $79.2\pm0.3$ & $68.6\pm0.2$ / \textcolor{Blue}{$\mathbf{87.7\pm0.2}$}\\
\midrule[0.5pt]
\end{tabular}
\end{adjustbox}
\caption{Comparison with other SOTA methods on Semi-Aves dataset. The best results are marked in \textcolor{Red}{red}, and the second best results are marked in \textcolor{Blue}{blue}.}
\label{sota_aves}
\end{minipage}

\hspace{0.1cm}
\makeatletter\def\@captype{table}
\begin{minipage}[t]{0.5\textwidth}
\centering
\begin{adjustbox}{width=1\linewidth,center}
\begin{tabular}{c l|c|c|c}
\toprule[1pt]
\multirow{2}{*}{}&\multirow{2}{*}{\textbf{Method}} & \textbf{From Scratch} ($\%$) &
\textbf{From ImageNet} ($\%$) & \textbf{From iNat} ($\%$)
\cr\cline{3-5} && Top-1 / Top-5& Top-1 / Top-5& Top-1 / Top-5\\
\midrule[1pt]
&Baseline & $31.0\pm0.4$ / $54.7\pm0.8$ & $53.8\pm0.4$ / $80.0\pm0.4$ & $52.4\pm0.6$ / $79.5\pm0.5$\\
\midrule[0.5pt]
\midrule[0.5pt]
\multirow{6}{*}{\begin{rotate}{90}
$U_{in}$
\end{rotate}} 

&Pseudo-Label & $19.4\pm0.4$ / $43.2\pm1.5$& $51.5\pm1.2$ / $81.2\pm0.2$& $49.5\pm0.4$ / $78.5\pm0.2$\\
&CPL & $31.4\pm0.6$ / $55.0\pm0.6$& $53.7\pm0.2$ / $80.2\pm0.1$& $53.3\pm0.5$ / $80.0\pm0.5$\\
&FixMatch & $32.2\pm1.0$ / $57.0\pm1.2$ & $56.3\pm0.5$ / $80.4\pm0.5$ & $58.7\pm0.7$ / $81.7\pm0.2$\\
&Self-Training & $32.7\pm0.2$ / $56.9\pm0.2$ & $56.9\pm0.3$ / $81.7\pm0.2$ & $55.7\pm0.3$ / $82.3\pm0.2$\\
&MoCo & $33.6\pm0.2$ / $59.4\pm0.3$ & $55.2\pm0.2$ / $82.9\pm0.2$ & $52.5\pm0.4$ / $79.5\pm0.2$\\
&MoCo + Self-Training & $39.4\pm0.3$ / \textcolor{red}{$\mathbf{64.4\pm0.5}$} & \textcolor{Blue}{$\mathbf{58.2\pm0.5}$} / $84.4\pm0.2$ & $55.2\pm0.5$ / \textcolor{Blue}{$\mathbf{82.9\pm0.2}$}\\
\midrule[0.5pt]
&\textbf{RelMatch} & $38.4\pm0.4$ / $61.4\pm0.3$ & $57.5\pm0.3$ / $83.4\pm0.3$ & \textcolor{Blue}{$\mathbf{56.8\pm0.3}$} / $82.5\pm0.3$\\
&\textbf{RelMatch + MoCo} & \textcolor{red}{$\mathbf{39.9\pm0.6}$} / \textcolor{red}{$\mathbf{64.4\pm0.6}$} & \textcolor{red}{$\mathbf{58.5\pm0.4}$} / \textcolor{red}{$\mathbf{84.7\pm0.2}$} & $56.3\pm0.5$ / $81.1\pm0.6$\\
&\textbf{RelMatch + FixMatch} & \textcolor{Blue}{$\mathbf{38.9\pm0.5}$} / \textcolor{Blue}{$\mathbf{63.2\pm0.5}$} & \textcolor{Blue}{$\mathbf{58.2\pm0.3}$} / \textcolor{Blue}{$\mathbf{84.6\pm0.4}$} & \textcolor{red}{$\mathbf{60.1\pm0.3}$} / \textcolor{red}{$\mathbf{83.3\pm0.2}$}\\
\midrule[0.5pt]
\midrule[0.5pt]
\multirow{9}{*}{\begin{rotate}{90}
$U_{in}+U_{out}$
\end{rotate}} 
&Pseudo-Label & $15.2\pm1.0$ / $40.6\pm1.2$ & $52.4\pm0.2$ / $80.4\pm0.5$ & $49.9\pm0.2$ / $78.5\pm0.3$\\
&CPL & $30.8\pm0.1$ / $54.4\pm0.3$ & $54.2\pm0.2$ / $79.9\pm0.2$ & $53.6\pm0.3$ / $79.9\pm0.2$\\
&FixMatch & $25.2\pm0.3$ / $50.2\pm0.8$ & $51.2\pm0.6$ / $77.6\pm0.3$ & $53.1\pm0.8$ / $79.9\pm0.1$\\
&Self-Training & $32.5\pm0.5$ / $56.3\pm0.3$ & $55.7\pm0.3$ / $81.0\pm0.2$ & $55.2\pm0.2$ / \textcolor{Blue}{$\mathbf{82.0\pm0.3}$}\\
&MoCo & $44.6\pm0.4$ / $72.6\pm0.5$ & $52.9\pm0.3$ / $81.2\pm0.1$ & $51.0\pm0.2$ / $78.5\pm0.3$\\
&MoCo + Self-Training & \textcolor{Blue}{$\mathbf{48.6\pm0.3}$}/ \textcolor{Blue}{$\mathbf{74.7\pm0.2}$} & $55.9\pm0.1$ / \textcolor{Blue}{$\mathbf{82.9\pm0.2}$} & $54.0\pm0.2$ / $81.3\pm0.3$\\
\midrule[0.5pt]
&\textbf{RelMatch} & $39.6\pm0.5$ / $68.8\pm0.7$ & \textcolor{Blue}{$\mathbf{57.6\pm0.2}$} / \textcolor{red}{$\mathbf{83.7\pm0.4}$} & \textcolor{red}{$\mathbf{57.1\pm0.5}$} / \textcolor{red}{$\mathbf{83.4\pm0.4}$}\\
&\textbf{RelMatch + MoCo} & \textcolor{red}{$\mathbf{49.9\pm0.6}$} / \textcolor{red}{$\mathbf{74.8\pm0.2}$} & \textcolor{red}{$\mathbf{58.0\pm0.2}$} / $82.7\pm0.3$ & \textcolor{Blue}{$\mathbf{56.3\pm0.5}$} / $81.2\pm0.3$\\
&\textbf{RelMatch + FixMatch} & $34.7\pm0.3$ / $56.8\pm0.2$ & $54.5\pm0.5$ / $79.1\pm0.4$ & $55.8\pm0.4$ / $81.6\pm0.4$\\
\midrule[0.5pt]
\end{tabular}
\end{adjustbox}
\caption{Comparison with other SOTA methods on Semi-Fungi dataset. The best results are marked in \textcolor{Red}{red}, and the second best results are marked in \textcolor{Blue}{blue}.}
\label{sota_fungi}
\end{minipage}
\end{adjustbox}
\vspace{-0.3cm}
\end{table*}


\subsection{Baseline Methods}

To demonstrate the superiority of RelMatch, we include following methods as baseline for comparison:

\noindent\textbf{(1) Vanilla Supervised Baseline}: The model is trained with only labeled data $X$ and the Corss-Entropy loss.

\noindent\textbf{(2) Pseudo-Labeling}~\cite{lee2013pseudo}: Following the implementation settings in~\cite{oliver2018realistic}, labeled data and unlabeled data are $1:1$ sampled. The method selects unlabeled samples with maximum prediction probabilities greater than a pre-defined threshold, and then generates one-hot pseudo-labels for them by predictions from the model itself.

\noindent\textbf{(3) Curriculum Pseudo-Labeling (CPL)}~\cite{cascante2020curriculum}: Curriculum labeling multiplies the model training to several phases. At each training phase, the model is trained from scratch with only currently labeled set in a supervised manner. And after each phase, a certain proportion of unlabeled samples with highest predictions will be pseudo-labeled and added into the labeled set. The whole training process will keep iterating until all samples are added into the labeled set.

\noindent\textbf{(4) FixMatch}~\cite{sohn2020fixmatch}: FixMatch combines the idea of pseudo-labeling and consistency regularization. For unlabeled data, pseudo-labels are given by their weakly augmented versions and are used to supervise their strong augmented versions.

\noindent\textbf{(5) Self-Training}~\cite{su2021realistic}: ``Self-Training'' is a widely used term, and in~\cite{su2021realistic} it refers specifically to a knowledge distillation~\cite{hinton2015distilling} based procedure. Firstly, a vanilla supervised model is trained with only labeled data $X$ to be the teacher model. And then, a student model is supervised by cross entropy loss on both labeled data and unlabeled data. Labels of unlabeled data are obtained by teacher model's prediction results.

\noindent\textbf{(6) MoCo}~\cite{he2020momentum}: As claimed in~\cite{chen2020big}, self-supervised model can also be a good semi-supervised learner. Restricted by limited computational resource, we adopt MoCo to train the image encoder due to its independence of large batch size. And the whole training process consist of unsupervised pre-training on unlabeled data followed by supervised fine-tuning on labeled data.

Additionally, \cite{su2021realistic} shows that combination of various SSL methods leads to further performance boosting. In this paper, we also include \textbf{(7) Self-Training + MoCo} as a baseline model. 

\subsection{Implementation Details}

For fair comparisons, we follow the experiment settings in~\cite{su2021realistic}. We use ResNet50~\cite{he2016deep} as the backbone network for all experiments. Input images are random-resize-cropped to $224\times224$ during training and simply resized to $224\times224$ during testing. We use SGD with a momentum of $0.9$ and a cosine learning rate decay schedule~\cite{loshchilov2016sgdr} for optimization. To perform comprehensive comparisons with state-of-the-art methods, we train RelMatch from scratch, ImageNet~\cite{russakovsky2015imagenet} pre-trained model, and iNaturalist $2018$ (iNat)~\cite{van2018inaturalist} pre-trained model. Note that, iNat is a large-scale fine-grained dataset that contains $8142$ species with no overlapping to Semi-Aves and Semi-Fungi. Our model is trained for $100k$ iterations and $50k$ iterations for training from scratch and pre-trained models, which is approximately aligned with the training setting of FixMatch in~\cite{su2021realistic}. Besides, the learning rate and weight decay are set to be $\{0.01, 0.001\}$ for training from scratch and $\{0.001, 0.0001\}$ for training from pre-trained models. The batch size is set as $32$ with $\mu=10$. For all experiments, we show the mean value and the standard deviation of $5$ independent runs.

\subsection{Comparison with SOTA Methods}

We show comparison results on Semi-Aves and Semi-Fungi datasets in Table~\ref{sota_aves} and Table~\ref{sota_fungi} respectively. When models are trained from scratch, MoCo~\cite{he2020momentum} shows the great advantage of leveraging abundant data. It achieves impressive performance on $U_{in}$ and obtains further boosting with the participation of $U_{out}$. While other semi-supervised methods are less effective. Instead of directly predicting pseudo-labels, with the label transfer strategy, the proposed RelMatch is able to generate more reliable pseudo-labels without pre-training. Therefore, we surpass MoCo with a large margin on $U_{in}$ and obtain the best performance when coupling with MoCo under both two settings. 

Under the setting of transfer learning, FixMatch~\cite{sohn2020fixmatch} tends to be the dominant technique with in-distribution data. However, it applies pseudo-labeling as the key component and is doomed to suffer performance degradation when faced with out-of-distribution unlabeled data. In comparison, Self-Training~\cite{su2021realistic} is much more robust to novel categories, despite a slight degradation. On the contrary, the proposed RelMatch with shared label space can easily obtain state-of-the-art results on its own with out-of-distribution data. And RelMatch also significantly outperforms all comparison methods when being equipped with FixMatch. In general, the prominent merits of RelMatch are three-fold: (i) through building a shared label space and enabling inter-sample consistency, RelMatch can leverage out-of-distribution data to mine useful knowledge, (ii) through the label transfer strategy, RelMatch can generate more reliable pseudo-labels, and (iii) as a higher-level constraint, RelMatch can offer complementary supervising signals and boost previous SSL techniques. 

\subsection{Ablation Studies}

\subsubsection{Different Relation-based Variants}

To discuss the effectiveness of each component upon the relation-based pseudo-labeling, we conduct ablation studies about its variants introduced above:\textbf{(A) Relation-based Pseudo-Labeling} as introduced in~\ref{relation_pl}, \textbf{(B) Triplet Consistency Regularization} as introduced in~\ref{triplet_cr}, and \textbf{(C) Label Transfer (RelMatch)} as introduced in~\ref{label_transfer}. As shown in Tabel~\ref{variants}, all experimental results are obtained by training the model from scratch. One important discovery is that all these relation-based techniques successfully leverage out-of-distribution unlabeled data and overcome the performance degradation. In addition, with the proposed \emph{triplet consistency} regularization and label transfer strategy, RelMatch delivers greater performance gains when compared with the naive relation-based pseudo-labeling, which indicates the superiorities brought by the inter-sample consistency and the reliable pseudo-labels via label transfer that enable better alignment.

\begin{table}[htbp]
\renewcommand{\arraystretch}{1.3}
\centering
\begin{adjustbox}{width=1\linewidth,center}
\begin{tabular}{c|c c|c c}
\toprule[1pt]
\multirow{2}{*}{\textbf{Variant}} & \multicolumn{2}{c|}{\textbf{Semi-Aves} ($\%$)} &
\multicolumn{2}{c}{\textbf{Semi-Fungi} ($\%$)}
\cr\cline{2-5} & in-distribution & out-of-distribution & in-distribution & out-of-distribution\\
\midrule[1pt]
A & $27.9\pm0.4$ & $28.1\pm0.4$ & $34.2\pm0.3$ & $34.5\pm0.3$\\
B & $30.6\pm0.2$ & $31.2\pm0.3$ & $35.6\pm0.5$ & $36.2\pm0.3$\\
C & $\mathbf{34.5\pm0.2}$ & $\mathbf{36.2\pm0.2}$ & $\mathbf{38.4\pm0.4}$ & $\mathbf{39.6\pm0.5}$\\
\midrule[0.5pt]
\end{tabular}
\end{adjustbox}
\caption{Ablation studies about different relation-based variants.}
\label{variants}
\end{table}
\vspace{-0.5cm}

\begin{table}[htbp]
\renewcommand{\arraystretch}{1.3}
\centering
\begin{adjustbox}{width=1\linewidth,center}
\begin{tabular}{c|c c|c c}
\toprule[1pt]
\multirow{2}{*}{\textbf{Depth}} & \multicolumn{2}{c|}{\textbf{Semi-Aves} ($\%$)} &
\multicolumn{2}{c}{\textbf{Semi-Fungi} ($\%$)}
\cr\cline{2-5} & in-distribution & out-of-distribution & in-distribution & out-of-distribution\\
\midrule[1pt]
2 & $24.3\pm0.4$ & $25.4\pm0.3$ & $32.0.8\pm0.4$ & $32.5\pm0.3$\\
3 & $29.2\pm0.3$ & $30.1\pm0.3$ & $33.8\pm0.2$ & $34.2\pm0.3$\\
4 & $33.9\pm0.2$ & $36.0\pm0.3$  & $36.2\pm0.5$ & $37.8\pm0.4$\\
5 & $\mathbf{34.5\pm0.2}$ & $\mathbf{36.2\pm0.2}$ & $37.8\pm0.4$ & $38.9\pm0.4$\\
6 & N/A & N/A & $38.0\pm0.4$ & $39.3\pm0.4$\\
7 & N/A & N/A & $\mathbf{38.4\pm0.4}$ & $\mathbf{39.6\pm0.5}$\\
\midrule[0.5pt]
\end{tabular}
\end{adjustbox}
\caption{Ablation studies about the depths of phylogenetic tree that participate model training. For ``Aves'' and ``Fungi'', the depths of their phylogenetic tree is $5$ and $7$, respectively.}
\label{tree_depth}
\end{table}
\vspace{-0.6cm}

\begin{table*}[t]
\begin{adjustbox}{width=1\linewidth,center}
\makeatletter\def\@captype{table}
\begin{minipage}[t]{0.67\linewidth}
\begin{adjustbox}{width=1\linewidth,center}
\centering
\begin{tabular}{c|c|c c|c c}
\toprule[1pt]
\multirow{2}{*}{\textbf{Sample Pair}}&\multirow{2}{*}{\textbf{Pseudo-labeling}} & \multicolumn{2}{c|}{\textbf{Semi-Aves} ($\%$)} &
\multicolumn{2}{c}{\textbf{Semi-Fungi} ($\%$)}
\cr\cline{3-6} & & in-distribution & out-of-distribution & in-distribution & out-of-distribution\\
\midrule[1pt]
$U$ versus $U$ & Label Prediction & $75.8\pm1.6$ & $70.1\pm1.8$ & $53.5\pm2.6$ & $49.0\pm3.1$ \\
$X$ versus $U$ & Label Prediction & $78.9\pm1.4$ & $76.7\pm1.7$ & $58.5\pm1.9$ & $55.4\pm2.6$ \\
$X$ versus $U$ & Label Transfer & $\mathbf{82.5\pm0.7}$ & $\mathbf{82.3\pm0.7}$ & $\mathbf{69.8\pm0.9}$ & $\mathbf{67.7\pm1.5}$ \\
\midrule[0.5pt]
\end{tabular}
\end{adjustbox}
\caption{Ablation studies about different pseudo-labeling methods against quality degradation against in-distribution and out-of-distribution unlabeled data.}
\label{relation_degradation}
\end{minipage}
\hspace{0.1cm}
\makeatletter\def\@captype{table}
\begin{minipage}[t]{0.3\textwidth}
\centering
\begin{adjustbox}{width=1\linewidth,center}
\begin{tabular}{c|c|c}
\toprule[1pt]
$\mathbf{U}$ & \textbf{Semi-Aves} ($\%$) & \textbf{Semi-Fungi} ($\%$)\\
\midrule[1pt]
N/A & $20.6\pm0.4$ & $31.0\pm0.4$\\
$U_{in} + U_{out}$ & $31.2\pm0.3$ & $36.2\pm0.3$\\
\midrule[0.5pt]
iNat & $28.2\pm0.7$ & $35.2\pm1.1$\\
ImageNet & $25.7\pm1.5$ & $33.3\pm1.4$\\
\midrule[0.5pt]
\end{tabular}
\end{adjustbox}
\caption{Ablation studies about different unlabeled datasets.}
\label{other_unlabeled_data}
\end{minipage}
\end{adjustbox}
\vspace{-0.2cm}
\end{table*}

\subsubsection{Depth of The Phylogenetic Tree}

The phylogenetic tree plays the most fundamental role in the proposed method. In this section, we conduct ablation studies about different tree depths. Specifically, we progressively reduce the depth of the tree from leaves to the root. All mentioned models are trained from scratch. As shown in Table~\ref{tree_depth}, the effect of tree depth is quite straightforward -- the model performance keeps boosting as the leveraged tree structure goes deeper. It is easy to understand that a deeper tree structure can provide more fine-grained supervision information.

\vspace{-0.2cm}
\subsubsection{Affect of Training Iterations}

To better understand the relation between model performances and numbers of training iterations, we conduct experiments with iteration numbers of $\{1K$, $5K$, $10K$, $15K$, $20K\}$. All models are trained from scratch in this section. As shown in Figure~\ref{fig:figure4}, we can conclude that model performances keep increasing with the increase of training iteration numbers and the margin of improvement will gradually decrease. Note that, as the number of iterations increases, the model will benefit more from the large amount of out-of-distribution data.

\begin{figure}[t]
\centering
\includegraphics[width=1\linewidth]{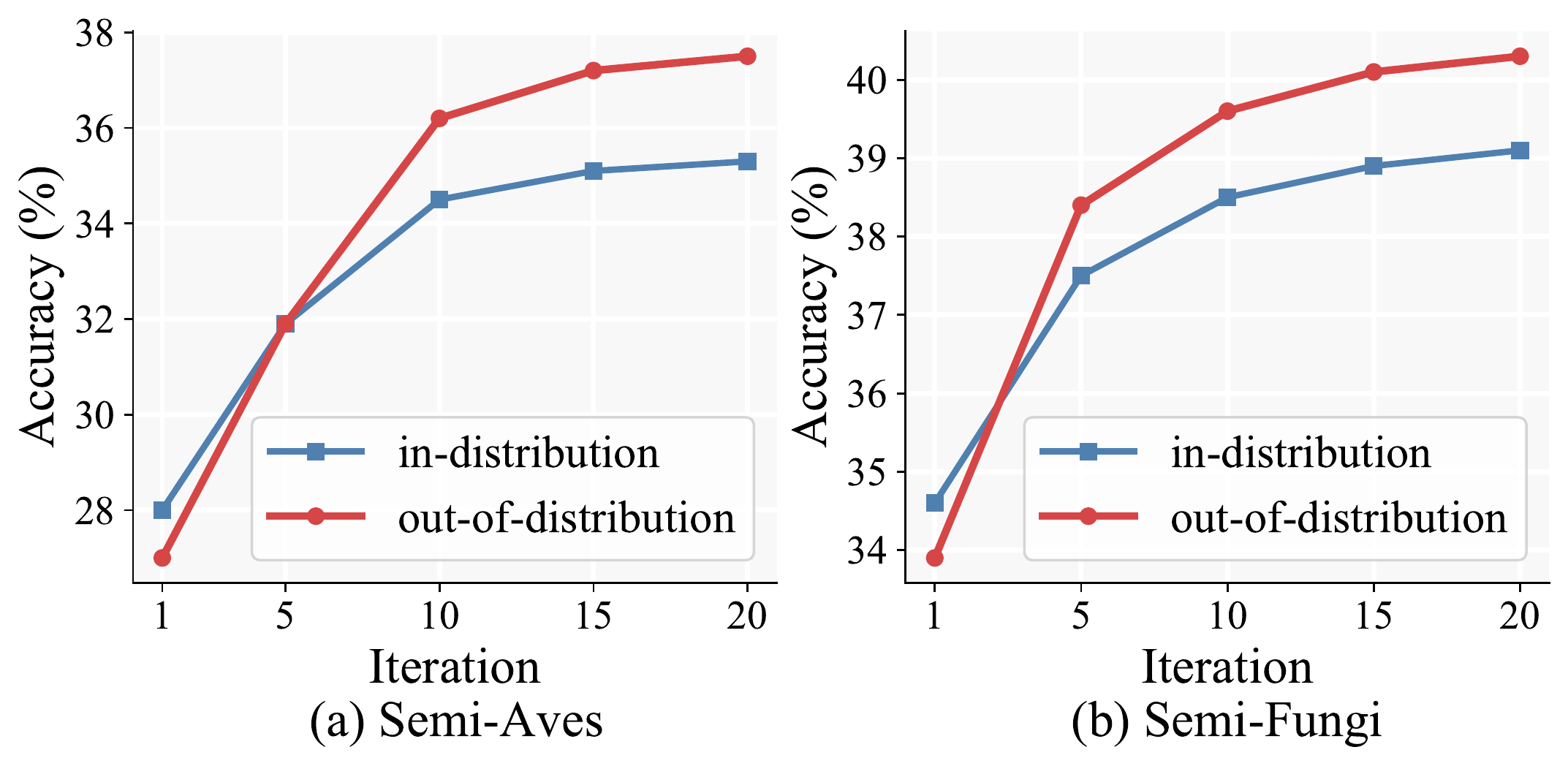}
\vspace{-0.6cm}
\caption{Ablation study on number of iterations.}
\label{fig:figure4}
\vspace{-0.3cm}
\end{figure}

\subsection{Discussion}

\noindent\textbf{How could RelMatch improve the quality of pseudo-labels?} 

To answer this question, we re-split the labeled set of Semi-Aves and Semi-Fungi to quantitatively evaluate pseudo-label quality. Specifically, we divide the $200$ labeled categories in half, then we regard one half as in-distribution categories and the other half as out-of-distribution categories. With the divided categories, we can split the training set and the testing set accordingly. We train the model on the in-distribution training set and test on the in-distribution and out-of-distribution testing set separately. As shown in Table~\ref{relation_degradation}, when the model predicts relations of two unlabeled samples, the prediction accuracy only drops about $5\%$ on out-of-distribution data, which solidly justifies the superiority of our relation-based prediction. When it comes to predicting relations between labeled and unlabeled data, the prediction accuracy naturally improves. Note that, this is not our contribution but just preparation for label transfer. And when we replace directly predicting relations with the proposed label transfer strategy via the \emph{triplet consistency}, the prediction accuracy obtains significant improvement, and the gap between in-distribution and out-of-distribution accuracy continues to decrease, which demonstrates that the label transfer strategy itself does make pseudo-labels more reliable.

\noindent\textbf{Can the learned tree structure generalize to other hierarchical data?}

Till now we have only verified that our model can handle novel categories belonging to the same underlying tree. Here we devolve further and ask -- can the learned tree structure generalize to fine-grained categories that are beyond the original phylogenetic tree of Aves/Fungi? To answer that, we experiment with iNat~\cite{van2018inaturalist} and ImageNet~\cite{russakovsky2015imagenet} datasets as unlabeled data. More specifically, we use the variant of RelMatch without label transfer, \emph{i.e.}, the \emph{triplet consistency} regularization introduced in Section~\ref{triplet_cr}. Results are shown in Table~\ref{other_unlabeled_data}. It is interesting to see that these datasets do seem to contribute meaningfully and yield performance gains, despite with a smaller margin when compared with unlabeled data from the same latent tree. We attribute this to the proposed relation-based prediction still being able to be exploited as a higher-level constraint in other hierarchical datasets.

\section{Conclusion}

In this paper, we approached the fine-grained visual classification problem via a semi-supervised learning setting. Our key contribution is how best to train with out-of-distribution data. Our solution rests with leveraging the underlying tree structure of fine-grained categories to build a relation-based common label space. We further introduced a \emph{triplet consistency} regularization to help with in-distribution and out-of-distribution alignment. We evaluated the proposed method on semi-supervised FGVC benchmark datasets and reported state-of-the-art results.

{\small
\bibliographystyle{ieee_fullname}
\bibliography{egbib}
}

\end{document}